\documentclass[twoside]{article}
\usepackage[accepted]{aistats2015}
\usepackage{algorithm}
\usepackage[noend]{algorithmic}
\usepackage{amsmath}
\usepackage{amssymb}
\usepackage{amsthm}
\usepackage{bbm}
\usepackage{bm}
\usepackage{color}
\usepackage{dsfont}
\usepackage{enumerate}
\usepackage{graphicx}
\usepackage{subfigure}
\usepackage{times}

\usepackage[usenames,dvipsnames]{xcolor}
\usepackage[bookmarks=false]{hyperref}
\hypersetup{
  pdftex,
  pdffitwindow=true,
  pdfstartview={FitH},
  pdfnewwindow=true,
  colorlinks,
  linktocpage=true,
  linkcolor=Green,
  urlcolor=Green,
  citecolor=Green
}
\usepackage[backgroundcolor=White,textwidth=\marginparwidth]{todonotes}
\newcommand{\commentout}[1]{}
\newcommand{\junk}[1]{}
\newcommand{\etal}{\emph{et al.}}

\newcommand{\combucb}{{\tt CombUCB1}}

\newtheorem{theorem}{Theorem}

\newtheorem{lemma}{Lemma}
\newtheorem{proposition}{Proposition}
\newtheorem{claim}{Claim}

\newcommand{\bx}{{\bf x}}

\newcommand{\cE}{\mathcal{E}}
\newcommand{\ccE}{\overline{\cE}}
\newcommand{\cF}{\mathcal{F}}

\newcommand{\realset}{\mathbb{R}}

\newcommand{\abs}[1]{\left|#1\right|}

\newcommand{\E}[2]{\mathbb{E}_{#2} \! \left[#1\right]}
\newcommand{\EE}[1]{\mathbb{E} \left[#1\right]}

\newcommand{\I}[1]{\mathds{1} \! \left\{#1\right\}}

\newcommand{\set}[1]{\left\{#1\right\}}
\newcommand{\cset}[2]{\left\{#1:#2\right\}}

\DeclareMathOperator*{\argmax}{arg\,max\,}

\begin{document}

\runningtitle{Tight Regret Bounds for Stochastic Combinatorial Semi-Bandits}
\runningauthor{Branislav Kveton, Zheng Wen, Azin Ashkan, and Csaba Szepesv\'ari}

\twocolumn[
\aistatstitle{Tight Regret Bounds for Stochastic Combinatorial Semi-Bandits}
\vspace{-0.2in}
\aistatsauthor{Branislav Kveton \And Zheng Wen}
\aistatsaddress{Adobe Research \\ San Jose, CA \\ \em{kveton@adobe.com} \And
Yahoo Labs \\ Sunnyvale, CA \\ \em{zhengwen@yahoo-inc.com}}
\vspace{-0.15in}
\aistatsauthor{Azin Ashkan \And Csaba Szepesv\'ari}
\aistatsaddress{Technicolor Labs \\ Los Altos, CA \\ \em{azin.ashkan@technicolor.com} \And
Department of Computing Science \\ University of Alberta \\ \em{szepesva@cs.ualberta.ca}}
\vspace{-0.1in}]

\begin{abstract}
A stochastic combinatorial semi-bandit is an online learning problem where at each step a learning agent chooses a subset of ground items subject to constraints, and then observes stochastic weights of these items and receives their sum as a payoff. In this paper, we close the problem of computationally and sample efficient learning in stochastic combinatorial semi-bandits. In particular, we analyze a UCB-like algorithm for solving the problem, which is known to be computationally efficient; and prove $O(K L (1 / \Delta) \log n)$ and $O(\sqrt{K L n \log n})$ upper bounds on its $n$-step regret, where $L$ is the number of ground items, $K$ is the maximum number of chosen items, and $\Delta$ is the gap between the expected returns of the optimal and best suboptimal solutions. The gap-dependent bound is tight up to a constant factor and the gap-free bound is tight up to a polylogarithmic factor.
\end{abstract}

%!TEX root = Paper.tex

\section{Introduction}
\label{sec:introduction}

A \emph{stochastic combinatorial semi-bandit} \cite{gai12combinatorial,chen13combinatorial} is an online learning problem where at each step a learning agent chooses a subset of ground items subject to combinatorial constraints, and then observes stochastic weights of these items and receives their sum as a payoff. The problem can be viewed as a learning variant of combinatorial optimization with a linear objective function and binary variables. Many classical combinatorial optimization problems have linear objectives \cite{papadimitriou98combinatorial}. Therefore, stochastic combinatorial semi-bandits have found many practical applications, such as learning spectrum allocations \cite{gai12combinatorial}, shortest paths \cite{gai12combinatorial}, routing networks \cite{kveton14matroid}, and recommendations \cite{kveton14matroid,kveton14learning}.

In our work, we study a variant of stochastic combinatorial semi-bandits where the learning agent has access to an \emph{offline optimization oracle} that can find the optimal solution for any weights of the items. We say that the problem is a $(L, K, \Delta)$ instance when $L$ is the cardinality of its ground set $E$, $K$ is the maximum number of chosen items, and $\Delta$ is the gap between the expected returns of the optimal and best suboptimal solutions. We also say that the problem is a $(L, K)$ instance if it is a $(L, K, \Delta)$ instance for some $\Delta$. Based on the existing bandit work \cite{auer02finitetime}, it is relatively easy to propose a UCB-like algorithm for solving our problem \cite{gai12combinatorial}, and we call this algorithm $\combucb$. $\combucb$ is a variant of ${\tt UCB1}$ that calls the oracle to find the optimal solution with respect to the upper confidence bounds on the weights of the items. Chen \etal~\cite{chen13combinatorial} recently showed that the $n$-step regret of $\combucb$ in any $(L, K, \Delta)$ stochastic combinatorial semi-bandit is $O(K^2 L (1 / \Delta) \log n)$.

Our main contribution is that we derive two upper bounds on the $n$-step regret of $\combucb$, $O(K L (1 / \Delta) \log n)$ and $O(\sqrt{K L n \log n})$. Both of these bounds are significant improvements over Chen \etal~\cite{chen13combinatorial}. Moreover, we prove two novel lower bounds, $\Omega(K L (1 / \Delta) \log n)$ and $\Omega(\sqrt{K L n})$, which match our upper bounds up to polylogarithmic factors. The consequence of these results is that $\combucb$ is \emph{sample efficient} because it achieves near-optimal regret. It is well known that $\combucb$ is also \emph{computationally efficient} \cite{gai12combinatorial}, it can be implemented efficiently whenever the offline optimization oracle is computationally efficient. So we close the problem of computationally and sample efficient learning in stochastic combinatorial semi-bandits, by showing that $\combucb$ has both properties. This problem is still open in the adversarial setting (Section~\ref{sec:related work}).

Our analysis is novel. It is based on the idea that the event that ``many'' items in a chosen suboptimal solution are not observed ``sufficiently often'' does not happen ``too often''. The reason is that this event happens for ``many'' items simultaneously. Therefore, when the event happens, the observation counters of ``many'' items increase. Based on this observation, we divide the regret associated with the event among ``many'' items, instead of attributing it separately to each item as in the prior work \cite{gai12combinatorial,chen13combinatorial}. This is the key step in our analysis that yields tight upper bounds.

Our paper is organized as follows. In Section~\ref{sec:setting}, we introduce our learning problem and the algorithm for solving it. In Section~\ref{sec:discussion}, we summarize our results. In Section~\ref{sec:K43 upper bounds}, we prove a $O(K^\frac{4}{3} L (1 / \Delta) \log n)$ upper bound on the regret of $\combucb$. In Section~\ref{sec:K upper bounds}, we prove a $O(K L (1 / \Delta) \log n)$ upper bound on the regret of $\combucb$. In Section~\ref{sec:lower bounds}, we prove gap-dependent and gap-free lower bounds. In Section~\ref{sec:experiments}, we evaluate $\combucb$ on a synthetic problem and show that its $n$-step regret grows as suggested by our gap-dependent upper bound. In Section~\ref{sec:related work}, we compare our results to prior work. In Section~\ref{sec:extensions}, we discuss extensions of our work. We conclude in Section~\ref{sec:conclusions}.

%!TEX root = Paper.tex

\section{Setting}
\label{sec:setting}

Formally, a \emph{stochastic combinatorial semi-bandit} is a tuple $B = (E, \Theta, P)$, where $E = \set{1, \dots, L}$ is a finite set of $L$ items, $\Theta \subseteq 2^E$ is a non-empty set of feasible subsets of $E$, and $P$ is a probability distribution over a \mbox{unit cube $[0, 1]^E$.} We borrow the terminology of combinatorial optimization and call $E$ the \emph{ground set}, $\Theta$ the \emph{feasible set}, and $A \in \Theta$ a \emph{solution}. The items in the ground set $E$ are associated with a vector of stochastic \emph{weights} $w \sim P$. The $e$-th entry of $w$, $w(e)$, is the weight of item $e$. The expected \mbox{weights of the} items are defined as $\bar{w} = \E{w}{w \sim P}$. The \emph{return} for choosing solution $A$ under the realization of the weights $w$ is:
\begin{align*}
  f(A, w) = \sum_{e \in A} w(e)\,.
\end{align*}
The maximum number of chosen items is defined as $K = \max_{A \in \Theta} \abs{A}$.

Let $(w_t)_{t = 1}^n$ be an i.i.d. sequence of $n$ weights drawn from $P$. At time $t$, the learning agents chooses solution $A_t \in \Theta$ based on its observations of the weights up to time $t$, gains $f(A_t, w_t)$, and observes the weights of all chosen items at time $t$, $\cset{(e, w_t(e))}{e \in A_t}$. The learning agent interacts with the environment $n$ times and its goal is to maximize its expected cumulative reward over this time. If the agent knew $P$ a priori, the optimal action would be to choose the \emph{optimal solution}\footnote{For simplicity of exposition, we assume that the optimal solution is unique.}:
\begin{align*}
  \textstyle
  A^\ast = \argmax_{A \in \Theta} f(A, \bar{w})
\end{align*}
at all steps $t$. The quality of the agent's policy is measured by its \emph{expected cumulative regret}:
\begin{align*}
  R(n) = \EE{\sum_{t = 1}^n R(A_t, w_t)}\,,
\end{align*}
where $R(A_t, w_t) = f(A^\ast, w_t) - f(A_t, w_t)$ is the regret of the agent at time $t$.

\begin{algorithm}[t]
  \caption{$\combucb$ for stochastic combinatorial semi-bandits.}
  \label{alg:ucb1}
  \begin{algorithmic}
    \STATE {\bf Input:} Feasible set $\Theta$
    \STATE
    \STATE // Initialization
    \STATE $(\hat{w}_1, t_0) \gets {\tt Init}(\Theta)$
    \STATE $T_{t_0 - 1}(e) \gets 1 \hspace{1.815in} \forall e \in E$
    \STATE
    \FORALL{$t = t_0, \dots, n$}
      \STATE // Compute UCBs
      \STATE $U_t(e) \gets \hat{w}_{T_{t - 1}(e)}(e) + c_{t - 1, T_{t - 1}(e)} \hspace{0.455in} \forall e \in E$
      \STATE \vspace{-0.08in}
      \STATE // Solve the optimization problem
      \STATE $A_t \gets \argmax_{A \in \Theta} f(A, U_t)$
      \STATE \vspace{-0.08in}
      \STATE // Observe the weights of chosen items
      \STATE Observe $\cset{(e, w_t(e))}{e \in A_t}$, where $w_t \sim P$
      \STATE \vspace{-0.08in}
      \STATE // Update statistics
      \STATE $T_t(e) \gets T_{t - 1}(e) \hspace{1.5in} \forall e \in E$
      \STATE $T_t(e) \gets T_t(e) + 1 \hspace{1.405in} \forall e \in {A_t}$
      \STATE $\displaystyle \hat{w}_{T_t(e)}(e) \gets
      \frac{T_{t - 1}(e) \hat{w}_{T_{t - 1}(e)}(e) + w_t(e)}{T_t(e)} \hspace{0.1in} \forall e \in {A_t}$
    \ENDFOR
  \end{algorithmic}
\end{algorithm}

\begin{algorithm}[t]
  \caption{${\tt Init}$: Initialization of $\combucb$.}
  \label{alg:initialization}
  \begin{algorithmic}
    \STATE {\bf Input:} Feasible set $\Theta$
    \STATE
    \STATE $\hat{w}(e) \gets 0 \hspace{2.04in} \forall e \in E$
    \STATE $u(e) \gets 1 \hspace{2.065in} \forall e \in E$
    \STATE $t \gets 1$
    \WHILE{$(\exists e \in E: u(e) = 1)$}
      \STATE $A_t \gets \argmax_{A \in \Theta} f(A, u)$
      \STATE Observe $\cset{(e, w_t(e))}{e \in A_t}$, where $w_t \sim P$
      \FORALL{$e \in A_t$}
        \STATE $\hat{w}(e) \gets w_t(e)$
        \STATE $u(e) \gets 0$
      \ENDFOR
      \STATE $t \gets t + 1$
    \ENDWHILE
    \STATE
    \STATE {\bf Output:}
    \STATE \quad Weight vector $\hat{w}$
    \STATE \quad First non-initialization step $t$
  \end{algorithmic}
\end{algorithm}

\subsection{Algorithm}
\label{sec:algorithm}

Gai \etal~\cite{gai12combinatorial} proposed a simple algorithm for stochastic combinatorial semi-bandits. The algorithm is motivated by ${\tt UCB1}$ \cite{auer02finitetime} and therefore we call it $\combucb$. At each time $t$, $\combucb$ consists of three steps. First, it computes the \emph{upper confidence bound (UCB)} on the expected weight of each item $e$:
\begin{align}
  U_t(e) = \hat{w}_{T_{t - 1}(e)}(e) + c_{t - 1, T_{t - 1}(e)}\,,
  \label{eq:UCB}
\end{align}
where $\hat{w}_s(e)$ is the average of $s$ observed weights of item $e$, $T_t(e)$ is the number of times that item $e$ is observed in $t$ steps, and:
\begin{align}
  c_{t, s} = \sqrt{\frac{1.5 \log t}{s}}
  \label{eq:confidence radius}
\end{align}
is the radius of a confidence interval around $\hat{w}_s(e)$ at time $t$ such that $\bar{w}(e) \in [\hat{w}_s(e) - c_{t, s}, \hat{w}_s(e) + c_{t, s}]$ holds with high probability. Second, $\combucb$ calls the \emph{optimization oracle} to solve the combinatorial problem on the UCBs:
\begin{align*}
  \textstyle
  A_t = \argmax_{A \in \Theta} f(A, U_t)\,.
\end{align*}
Finally, $\combucb$ chooses $A_t$, observes the weights of all chosen items, and updates the estimates of $\bar{w}(e)$ for these items. The pseudocode of $\combucb$ is in Algorithm~\ref{alg:ucb1}.

\subsection{Initialization}
\label{sec:initialization}

$\combucb$ is initialized by calling procedure ${\tt Init}$ (Algorithm~\ref{alg:initialization}), which returns two variables. The first variable is a weight vector $\hat{w} \in [0, 1]^E$, where $\hat{w}(e)$ is a single observation from the $e$-th marginal of $P$. The second variable is the number of initialization steps plus one.

The weight vector $\hat{w}$ is computed as follows. ${\tt Init}$ repeatedly calls the oracle $A_t = \argmax_{A \in \Theta} f(A, u)$ on a vector of auxiliary weights $u \in \set{0, 1}^E$, which are initialized to ones. When item $e$ is observed, we set the weight $\hat{w}(e)$ to the observed weight of the item and $u(e)$ to zero. ${\tt Init}$ terminates when $u(e) = 0$ for all items $e$. Without loss of generality, let's assume that each item $e$ is contained in at least one feasible solution. Then ${\tt Init}$ is guaranteed to terminate in at most $L$ iterations, because at least one entry of $u$ changes from one to zero after each call of the optimization oracle.

%!TEX root = Paper.tex

\section{Summary of Main Results}
\label{sec:discussion}

We prove three upper bounds on the regret of $\combucb$. Two bounds depend on the gap $\Delta$ and one is gap-free:
\begin{align*}
  \text{Theorem~\ref{thm:K43}}: & \qquad O(K^\frac{4}{3} L (1 / \Delta) \log n) \\
  \text{Theorem~\ref{thm:K}}: & \qquad O(K L (1 / \Delta) \log n) \\
  \text{Theorem~\ref{thm:gap-free}}: & \qquad O(\sqrt{K L n \log n})\,.
\end{align*}
Both gap-dependent bounds are major improvements over $O(K^2 L (1 / \Delta) \log n)$, the best known upper bound on the $n$-step regret of $\combucb$ \cite{chen13combinatorial}. The bound in Theorem~\ref{thm:K} is asymptotically tighter than the bound in Theorem~\ref{thm:K43}, but the latter is tighter for $K < (534 / 96)^3 < 173$.

One of the main contributions of our work is that we identify an algorithm for stochastic combinatorial semi-bandits that is both computationally and sample efficient. The following are our definitions of computational and sample efficiency. We say that the algorithm is \emph{computationally efficient} if it can be implemented efficiently whenever the offline variant of the problem can be solved computationally efficiently. The algorithm is \emph{sample efficient} if it achieves optimal regret up to polylogarithmic factors. Based on our definitions, $\combucb$ is both \emph{computationally} and \emph{sample efficient}. We state this result slightly more formally below.

\begin{theorem}
\label{thm:efficiency} $\combucb$ is computationally and sample efficient in any $(L, K)$ stochastic combinatorial semi-bandit where the offline optimization oracle $\arg\max_{A \in \Theta} f(A, w)$ can be implemented efficiently for any $w \in (\realset^+)^E$.
\end{theorem}
\begin{proof}
In each step $t$, $\combucb$ calls the oracle once, and all of its remaining operations are polynomial in $L$ and $K$. Therefore, $\combucb$ is guaranteed to be computationally efficient when the oracle is computationally efficient.

$\combucb$ is sample efficient because it achieves optimal regret up to polylogarithmic factors. In particular, the gap-dependent upper bound on the $n$-step regret of $\combucb$ in Theorem~\ref{thm:K} matches the lower bound in Proposition~\ref{prop:lower bound} up to a constant factor. In addition, the gap-free upper bound in Theorem~\ref{thm:gap-free} matches the lower bound in Proposition~\ref{prop:gap-free lower bound} up to a factor of $\sqrt{\log n}$.
\end{proof}

%!TEX root = Paper.tex

\section{$O(K^\frac{4}{3})$ Upper Bounds}
\label{sec:K43 upper bounds}

In this section, we prove two $O(K^\frac{4}{3} L (1 / \Delta) \log n)$ upper bounds on the $n$-step regret of $\combucb$. In Theorem~\ref{thm:K43 one gap}, we assume that the gaps of all suboptimal solutions are the same. In Theorem~\ref{thm:K43}, we relax this assumption.

The \emph{gap} of solution $A$ is $\Delta_A = f(A^\ast, \bar{w}) - f(A, \bar{w})$. The results in this section are presented for their didactic value. Their proofs are simple. Yet they illustrate the main ideas that lead to the tight regret bounds in Section~\ref{sec:K upper bounds}.

\begin{theorem}
\label{thm:K43 one gap} In any $(L, K, \Delta)$ stochastic combinatorial semi-bandit where $\Delta_{A} = \Delta$ for all suboptimal solutions $A$, the regret of $\combucb$ is bounded as:
\begin{align*}
  R(n) \leq K^\frac{4}{3} L \frac{48}{\Delta} \log n + \left(\frac{\pi^2}{3} + 1\right) K L\,.
\end{align*}
\end{theorem}

The proof of Theorem~\ref{thm:K43 one gap} relies on two lemmas. In the first lemma, we bound the regret associated with the initialization of $\combucb$ and the event that $\bar{w}(e)$ is outside of the high-probability confidence interval around $\hat{w}_{T_{t - 1}(e)}(e)$.

\begin{lemma}
\label{lem:regret decomposition} Let:
\begin{align}
  \cF_t = \set{\Delta_{A_t} \leq 2 \sum_{e \in \tilde{A}_t} c_{n, T_{t - 1}(e)}, \ \Delta_{A_t} > 0}
  \label{eq:suboptimality event}
\end{align}
be the event that suboptimal solution $A_t$ is ``hard to distinguish'' from $A^\ast$ at time $t$, where $\tilde{A}_t = A_t \setminus A^\ast$. Then the regret of $\combucb$ is bounded as:
\begin{align}
  R(n) \leq \EE{\hat{R}(n)} + \left(\frac{\pi^2}{3} + 1\right) K L\,,
  \label{eq:regret bound}
\end{align}
where:
\begin{align}
  \hat{R}(n) = \sum_{t = t_0}^n \Delta_{A_t} \I{\cF_t}.
  \label{eq:Rhat}
\end{align}
\end{lemma}
\begin{proof}
The claim is proved in Appendix~\ref{sec:proof regret decomposition}.
\end{proof}

Now we bound the regret corresponding to the events $\cF_t$, the items in a suboptimal solution are not observed ``sufficiently often'' up to time $t$. To bound the regret, we define two events:
\begin{align}
  G_{1, t} = \bigg\{
  & \text{at least $d$ items in $\tilde{A}_t$ were observed} \label{eq:event 1} \\
  & \text{at most $\alpha K^2 \frac{6}{\Delta_{A_t}^2} \log n$ times}
  \bigg\} \nonumber
\end{align}
and:
\begin{align}
  G_{2, t} = \bigg\{
  & \text{less than $d$ items in $\tilde{A}_t$ were observed} \label{eq:event 2} \\
  & \text{at most $\alpha K^2 \frac{6}{\Delta_{A_t}^2} \log n$ times}, \nonumber \\
  & \text{at least one item in $\tilde{A}_t$ was observed} \nonumber \\
  & \text{at most $\frac{\alpha d^2}{(\sqrt{\alpha} - 1)^2} \frac{6}{\Delta_{A_t}^2} \log n$ times}
  \bigg\}\,, \nonumber
\end{align}
where $\alpha \geq 1$ and $d > 0$ are parameters, which are chosen later. The event $G_{1, t}$ happens when ``many'' chosen items, at least $d$, are not observed ``sufficiently often'' up to time $t$, at most $\alpha K^2 \frac{6}{\Delta_{A_t}^2} \log n$ times.

Events $G_{1, t}$ and $G_{2, t}$ are obviously mutually exclusive. In the next lemma, we prove that these events are exhaustive when event $\cF_t$ happens. To prove this claim, we introduce new notation. We denote the set of items in $\tilde{A}_t$ that are not observed ``sufficiently often'' up to time $t$ by:
\begin{align*}
  S_t = \cset{e \in \tilde{A}_t}{T_{t - 1}(e) \leq \alpha K^2 \frac{6}{\Delta_{A_t}^2} \log n}\,.
\end{align*}

\begin{lemma}
\label{lem:K43 events} Let $\alpha \geq 1$, $d > 0$, and event $\cF_t$ happen. Then either event $G_{1, t}$ or $G_{2, t}$ happens.
\end{lemma}
%%%%%     Proof     %%%%%
\begin{proof}
By the definition of $S_t$, the following three events:\small
\begin{align*}
  G_{1, t} & = \set{|S_t| \geq d} \\
  G_{2, t} & = \set{|S_t| < d, \ \left[\exists e \in \tilde{A}_t: T_{t - 1}(e) \leq
  \frac{6 \alpha d^2 \log n}{(\sqrt{\alpha} - 1)^2 \Delta_{A_t}^2}\right]} \\
  \bar{G}_t & = \set{|S_t| < d, \ \left[\forall e \in \tilde{A}_t: T_{t - 1}(e) >
  \frac{6 \alpha d^2 \log n}{(\sqrt{\alpha} - 1)^2 \Delta_{A_t}^2}\right]}
\end{align*}\normalsize
are exhaustive and mutually exclusive. Therefore, to prove that either $G_{1, t}$ or $G_{2, t}$ happens, it suffices to show that $\bar{G}_t$ does not happen. Suppose that event $\bar{G}_t$ happens. Then by the assumption that $\cF_t$ happens and from the definition of $\bar{G}_t$, it follows that:
\begin{align*}
  \Delta_{A_t}
  & \leq 2 \sum_{e \in \tilde{A}_t} \sqrt{\frac{1.5 \log n}{T_{t - 1}(e)}} \\
  & = 2 \sum_{e \in \tilde{A}_t \setminus S_t} \sqrt{\frac{1.5 \log n}{T_{t - 1}(e)}} +
  2 \sum_{e \in S_t} \sqrt{\frac{1.5 \log n}{T_{t - 1}(e)}} \\
  & < 2 \underbrace{|\tilde{A}_t \setminus S_t|}_{\leq K}
  \sqrt{\frac{1.5 \log n}{\alpha K^2 \frac{6}{\Delta_{A_t}^2} \log n}} + {} \\
  & \quad\,\, 2 \underbrace{|S_t|}_{\leq d}
  \sqrt{\frac{1.5 \log n}{\frac{\alpha d^2}{(\sqrt{\alpha} - 1)^2} \frac{6}{\Delta_{A_t}^2} \log n}} \\
  & \leq \frac{\Delta_{A_t}}{\sqrt{\alpha}} + \frac{\Delta_{A_t} (\sqrt{\alpha} - 1)}{\sqrt{\alpha}}
  = \Delta_{A_t}\,.
\end{align*}
This is clearly a contradiction. Therefore, event $\bar{G}_t$ cannot happen; and either $G_{1, t}$ or $G_{2, t}$ happens.
\end{proof}

Now we are ready to prove Theorem~\ref{thm:K43 one gap}.

%%%%%     Proof     %%%%%
\begin{proof}
A detailed proof is in Appendix~\ref{sec:proof K43 one gap}. The key idea is to bound the number of times that events $G_{1, t}$ and $G_{2, t}$ happen in $n$ steps. Based on these bounds, the regret associated with both events is bounded as:
\begin{align*}
  \hat{R}(n) \leq
  \left(\frac{\alpha}{d} K^2 + \frac{\alpha d^2}{(\sqrt{\alpha} - 1)^2}\right) L \frac{6}{\Delta} \log n\,.
\end{align*}
Finally, we choose $\alpha = 4$ and $d = K^\frac{2}{3}$, and substitute the above upper bound into inequality~\eqref{eq:regret bound}.
\end{proof}

Theorem~\ref{thm:K43 one gap} can be generalized to the problems with different gaps. Let $\Delta_{e, \min}$ be the minimum gap of any suboptimal solution that contains item $e \in \tilde{E}$:
\begin{align}
  \Delta_{e, \min}
  & = \min_{A \in \Theta: e \in A, \Delta_A > 0} \Delta_A \label{eq:min gap} \\
  & = f(A^\ast, \bar{w}) - \max_{A \in \Theta: e \in A, \Delta_A > 0} f(A, \bar{w})\,, \nonumber
\end{align}
where $\tilde{E} = E \setminus A^\ast$ is the set of \emph{subptimal items}, the items that do not appear in the optimal solution. Then the regret of $\combucb$ is bounded as follows.

\begin{theorem}
\label{thm:K43} In any $(L, K)$ stochastic combinatorial semi-bandit, the regret of $\combucb$ is bounded as:
\begin{align*}
  R(n) \leq \sum_{e \in \tilde{E}} K^\frac{4}{3} \frac{96}{\Delta_{e, \min}} \log n +
  \left(\frac{\pi^2}{3} + 1\right) K L\,.
\end{align*}
\end{theorem}
%%%%%     Proof     %%%%%
\begin{proof}
A detailed proof is in Appendix~\ref{sec:proof K43}. The key idea is to define item-specific variants of events $G_{1, t}$ and $G_{2, t}$, $G_{e, 1, t}$ and $G_{e, 2, t}$; and associate $\frac{\Delta_{A_t}}{d}$ and $\Delta_{A_t}$ regret with $G_{e, 1, t}$ and $G_{e, 2, t}$, respectively. Then, for each item $e$, we order the events from the largest gap to the smallest, and show that the total regret is bounded as:
\begin{align*}
  \hat{R}(n) < \sum_{e \in \tilde{E}}
  \left(\frac{\alpha}{d} K^2 + \frac{\alpha d^2}{(\sqrt{\alpha} - 1)^2}\right)
  \frac{12}{\Delta_{e, \min}} \log n\,.
\end{align*}
Finally, we choose $\alpha = 4$ and $d = K^\frac{2}{3}$, and substitute the above upper bound into inequality~\eqref{eq:regret bound}.
\end{proof}

%!TEX root = Paper.tex

\section{$O(K)$ Upper Bounds}
\label{sec:K upper bounds}

In this section, we improve on the results in Section~\ref{sec:K43 upper bounds} and derive $O(K L (1 / \Delta) \log n)$ upper bounds on the $n$-step regret of $\combucb$. In Theorem~\ref{thm:K one gap}, we assume that the gaps of all suboptimal solutions are identical. In Theorem~\ref{thm:K}, we relax this assumption.

The key step in our analysis is that we define a cascade of infinitely-many mutually-exclusive events and then bound the number of times that these events happen when a suboptimal solution is chosen. The events are parametrized by two decreasing sequences of constants:
\begin{align}
  1 = \beta_0 > \beta_1 & > \beta_2 > \ldots > \beta_k > \ldots \label{eq:condition 1} \\
  \alpha_1 & > \alpha_2 > \ldots > \alpha_k > \ldots \label{eq:condition 2}
\end{align}
such that $\lim_{i \to \infty} \alpha_i = \lim_{i \to \infty} \beta_i = 0$. We define:
\begin{align*}
  m_{i, t} = \alpha_i \frac{K^2}{\Delta_{A_t}^2} \log n
\end{align*}
and assume that $m_{i, t} = \infty$ when $\Delta_{A_t} = 0$. The events at time $t$ are defined as:
\begin{align}
  G_{1, t} = \{
  & \text{at least $\beta_1 K$ items in $\tilde{A}_t$ were observed} \label{eq:event i} \\
  & \text{at most $m_{1, t}$ times}\}\,, \nonumber \\
  G_{2, t} = \{
  & \text{less than $\beta_1 K$ items in $\tilde{A}_t$ were observed} \nonumber \\
  & \text{at most $m_{1, t}$ times}, \nonumber \\
  & \text{at least $\beta_2 K$ items in $\tilde{A}_t$ were observed} \nonumber \\
  & \text{at most $m_{2, t}$ times}\}\,, \nonumber \\
  & \vdots \nonumber \\
  G_{i, t} = \{
  & \text{less than $\beta_1 K$ items in $\tilde{A}_t$ were observed} \nonumber \\
  & \text{at most $m_{1, t}$ times}, \nonumber \\
  & \dots, \nonumber \\
  & \text{less than $\beta_{i - 1} K$ items in $\tilde{A}_t$ were observed} \nonumber \\
  & \text{at most $m_{i - 1, t}$ times}, \nonumber \\
  & \text{at least $\beta_i K$ items in $\tilde{A}_t$ were observed} \nonumber \\
  & \text{at most $m_{i, t}$ times}\}\,, \nonumber \\
  & \vdots \nonumber
\end{align}
The following lemma establishes a sufficient condition under which events $G_{i, t}$ are exhaustive. This is the key step in the proofs in this section.

\begin{lemma}
\label{lem:complete system} Let $(\alpha_i)$ and $(\beta_i)$ be defined as in \eqref{eq:condition 1} and \eqref{eq:condition 2}, respectively; and let:
\begin{align}
  \sqrt{6} \sum_{i = 1}^\infty \frac{\beta_{i - 1} - \beta_i}{\sqrt{\alpha_i}} \leq 1\,.
  \label{eq:condition 3}
\end{align}
Let event $\cF_t$ happen. Then event $G_{i, t}$ happens for some $i$.
\end{lemma}
%%%%%     Proof     %%%%%
\begin{proof}
We fix $t$ such that $\Delta_{A_t} > 0$. Because $t$ is fixed, we use shorthands $G_i = G_{i, t}$ and $m_i = m_{i, t}$. Let:
\begin{align*}
  S_i = \cset{e \in \tilde{A}_t}{T_{t - 1}(e) \leq m_i}
\end{align*}
be the set of items in $\tilde{A}_t$ that are not observed ``sufficiently often'' under event $G_i$. Then event $G_i$ can be written as:
\begin{align*}
  \textstyle
  G_i = \left(\bigcap_{j = 1}^{i - 1} \set{|S_j| < \beta_j K}\right) \cap \set{|S_i| \geq \beta_i K}\,.
\end{align*}
As in Lemma~\ref{lem:K43 events}, we prove that event $G_i$ happens for some $i$ by showing that the event that none of our events happen cannot happen. Note that this event can be written as:
\begin{align*}
  \bar{G}
  & = \overline{\bigcup_{i = 1}^\infty G_i} \\
  & = \bigcap_{i = 1}^\infty \Bigg[\Bigg(\bigcup_{j = 1}^{i - 1} \set{|S_j| \geq \beta_j K}\Bigg) \cup
  \set{|S_i| < \beta_i K}\Bigg] \\
  & = \bigcap_{i = 1}^\infty \set{|S_i| < \beta_i K}\,.
\end{align*}
Let $\bar{S_i} = \tilde{A}_t \setminus S_i$ and $S_0 = \tilde{A}_t$. Then by the definitions of $\bar{S_i}$ and $S_i$, $\bar{S}_{i - 1} \subseteq \bar{S}_i$ for all $i > 0$. Furthermore, note that $\lim_{i \to \infty} m_i = 0$. So there must exist an integer $j$ such that $\bar{S}_i = \tilde{A}_t$ for all $i > j$, and $\tilde{A}_t = \bigcup_{i = 1}^\infty (\bar{S}_i \setminus \bar{S}_{i - 1})$. Finally, by the definition of $\bar{S}_i$, $T_{t - 1}(e) > m_i$ for all $e \in \bar{S}_i$. Now suppose that event $\bar{G}$ happens. Then:
\begin{align*}
  \sum_{e \in \tilde{A}_t} \frac{1}{\sqrt{T_{t - 1}(e)}}
  & < \sum_{i = 1}^\infty \sum_{e \in \bar{S}_i \setminus \bar{S}_{i - 1}} \frac{1}{\sqrt{m_i}} \\
  & = \sum_{i = 1}^\infty \frac{|\bar{S}_i \setminus \bar{S}_{i - 1}|}{\sqrt{m_i}} \\
  & < \sum_{i = 1}^\infty \frac{(\beta_{i - 1} - \beta_i) K}{\sqrt{m_i}}\,,
\end{align*}
where the last inequality is due to Lemma~\ref{lem:little helper} (Appendix~\ref{sec:lemmas}). In addition, let event $\cF_t$ happen. Then:
\begin{align*}
  \Delta_{A_t}
  & \leq 2 \sum_{e \in \tilde{A}_t} \sqrt{\frac{1.5 \log n}{T_{t - 1}(e)}} \\
  & < \sqrt{6 \log n} \sum_{i = 1}^\infty \frac{(\beta_{i - 1} - \beta_i) K}{\sqrt{m_i}} \\
  & = \Delta_{A_t} \sqrt{6} \sum_{i = 1}^\infty \frac{\beta_{i - 1} - \beta_i}{\sqrt{\alpha_i}}
  \leq \Delta_{A_t}\,,
\end{align*}
where the last inequality is due to our assumption in \eqref{eq:condition 3}. The above is clearly a contradiction. As a result, $\bar{G}$ cannot happen, and event $G_i$ must happen for some $i$.
\end{proof}

\begin{theorem}
\label{thm:K one gap} In any $(L, K, \Delta)$ stochastic combinatorial semi-bandit where $\Delta_{A} = \Delta$ for all suboptimal solutions $A$, the regret of $\combucb$ is bounded as:
\begin{align*}
  R(n) \leq K L \frac{267}{\Delta} \log n + \left(\frac{\pi^2}{3} + 1\right) K L\,.
\end{align*}
\end{theorem}
%%%%%     Proof     %%%%%
\begin{proof}
A detailed proof is in Appendix~\ref{sec:proof K one gap}. The key idea is to bound the number of times that event $G_{i, t}$ happens in $n$ steps for any $i$. Based on this bound, the regret due to all events $G_{i, t}$ is bounded as:
\begin{align*}
  \hat{R}(n) \leq
  K L \frac{1}{\Delta} \left[\sum_{i = 1}^\infty \frac{\alpha_i}{\beta_i}\right] \log n\,,
\end{align*}
where $\hat{R}(n)$ is defined in (\ref{eq:Rhat}). Finally, we choose $(\alpha_i)$ and $(\beta_i)$, and apply the above upper bound in inequality~\eqref{eq:regret bound}.
\end{proof}

Now we generalize Theorem~\ref{thm:K one gap} to arbitrary gaps.

\begin{theorem}
\label{thm:K} In any $(L, K)$ stochastic combinatorial semi-bandit, the regret of $\combucb$ is bounded as:
\begin{align*}
  R(n) \leq \sum_{e \in \tilde{E}} K \frac{534}{\Delta_{e, \min}} \log n +
  \left(\frac{\pi^2}{3} + 1\right) K L\,,
\end{align*}
where $\Delta_{e, \min}$ is the minimum gap of suboptimal solutions that contain item $e$, which is defined in \eqref{eq:min gap}.
\end{theorem}
%%%%%     Proof     %%%%%
\begin{proof}
A detailed proof is in Appendix~\ref{sec:proof K}. The key idea is to introduce item-specific variants of events $G_{i, t}$, $G_{e, i, t}$, and associate $\frac{\Delta_{A_t}}{\beta_i K}$ regret with each of these events. Then, for each item $e$, we order the events from the largest gap to the smallest, and show that the total regret is bounded as:
\begin{align*}
  \hat{R}(n) <
  \sum_{e \in \tilde{E}} K \frac{2}{\Delta_{e, \min}}
  \left[\sum_{i = 1}^\infty \frac{\alpha_i}{\beta_i}\right] \log n\,,
\end{align*}
where $\hat{R}(n)$ is defined in (\ref{eq:Rhat}). Finally, we choose $(\alpha_i)$ and $(\beta_i)$, and apply the above upper bound in inequality~\eqref{eq:regret bound}.
\end{proof}

We also prove a gap-free bound.

\begin{theorem}
\label{thm:gap-free} In any $(L, K)$ stochastic combinatorial semi-bandit, the regret of $\combucb$ is bounded as:
\begin{align*}
  R(n) \leq 47 \sqrt{K L n \log n} + \left(\frac{\pi^2}{3} + 1\right) K L\,.
\end{align*}
\end{theorem}
%%%%%     Proof     %%%%%
\begin{proof}
The proof is in Appendix~\ref{sec:proof gap-free}. The key idea is to decompose the regret of $\combucb$ into two parts, where the gaps are larger than $\epsilon$ and at most $\epsilon$. We analyze each part separately and then set $\epsilon$ to get the desired result.
\end{proof}

%!TEX root = Paper.tex

\section{Lower Bounds}
\label{sec:lower bounds}

In this section, we derive two lower bounds on the $n$-step regret in stochastic combinatorial semi-bandits. One of the bounds is gap-dependent and the other one is gap-free.

\begin{figure*}[t]
  \centering
  \hspace{-0.15in}
  \includegraphics[width=4.4in, bb=1in 4.25in 6.5in 6.5in]{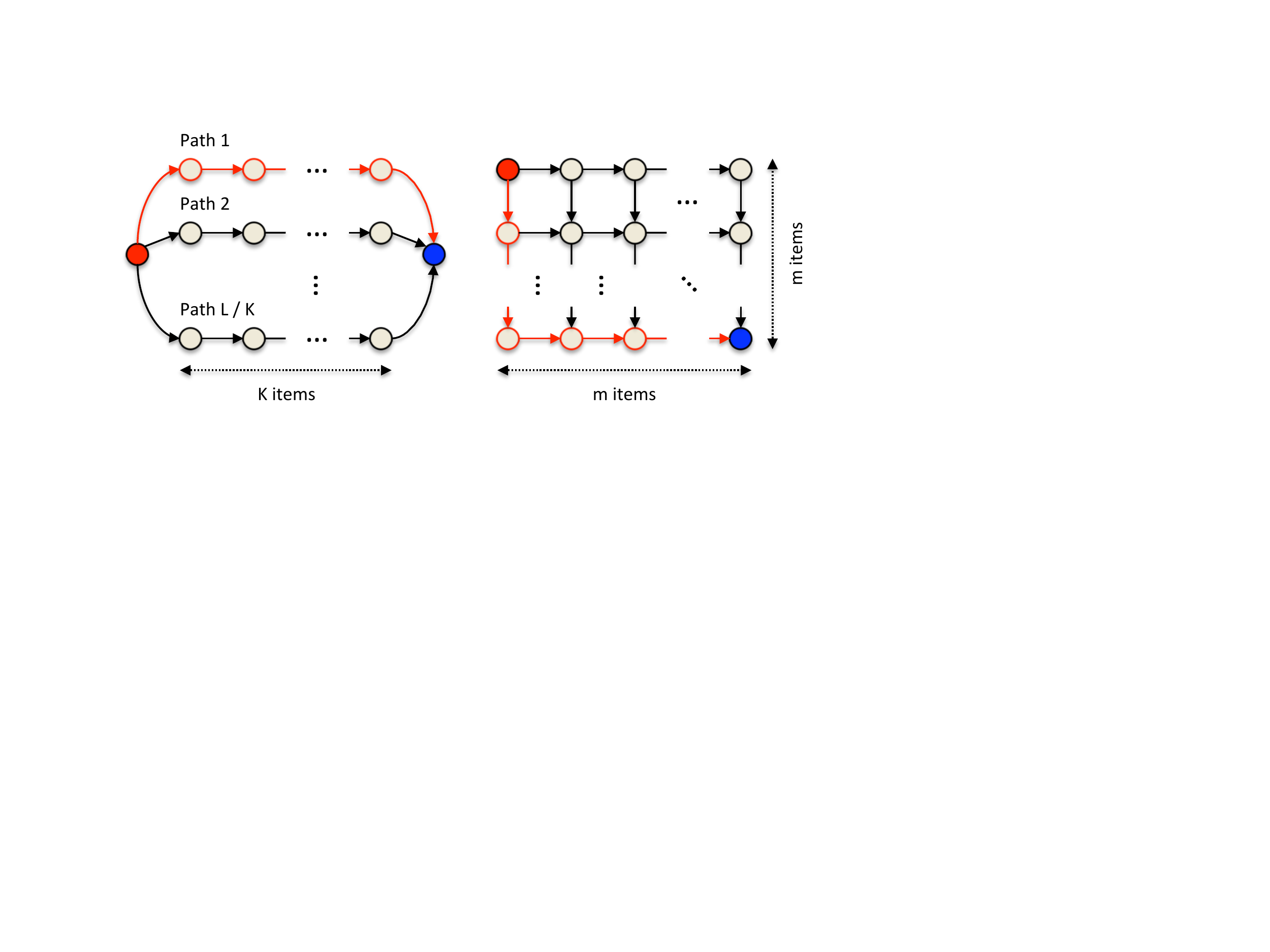}
  \raisebox{0.055in}{\includegraphics[width=2.4in, bb=2.75in 4.5in 5.75in 6.5in]{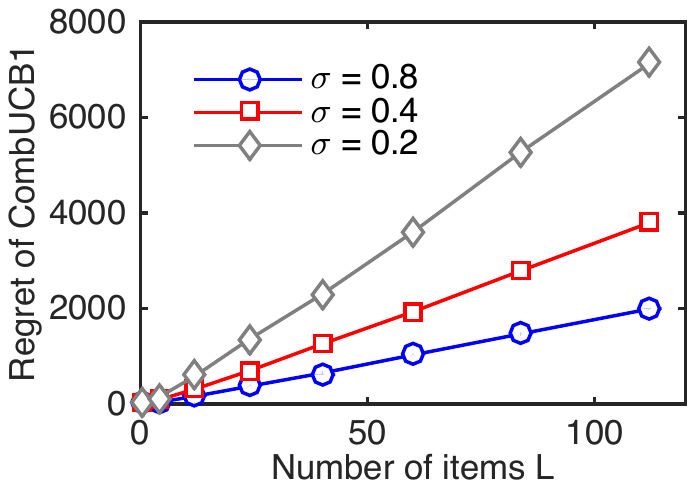}}
  \hspace{-0.4in} \makebox{} \\
  \hspace{0.09in} (a) \hspace{1.925in} (b) \hspace{2.37in} (c)
  \caption{{\bf a}. The $K$-path semi-bandit problem in Section~\ref{sec:lower bounds}. The red and blue nodes are the starting and end points of the paths, respectively. The optimal path is marked in red. {\bf b}. The grid-path problem in Section~\ref{sec:experiments}. The red and blue nodes are the starting and end points of the paths, respectively. The optimal path is marked in red. {\bf c}. The $n$-step regret of $\combucb$ on the grid-path problem.}
  \label{fig:main}
\end{figure*}

Our bounds are derived on a \emph{$K$-path semi-bandit} problem, which is illustrated in Figure~\ref{fig:main}a. The items in the ground set are $L$ path segments $E = \set{1, \dots, L}$. The feasible set $\Theta$ are $L / K$ paths, each of which contains $K$ unique items. Specifically, path $j$ contains items $(j - 1) K + 1, \dots, j K$. Without loss of generality, we assume that $L / K$ is an integer. The probability distribution $P$ over the weights of the items is defined as follows. The weights of the items in the same path are equal. The weights of the items in different paths are distributed independently. The weight of item $e$ is a Bernoulli random variable with mean:
\begin{align*}
  \bar{w}(e) =
  \begin{cases}
    0.5 & \text{item $e$ belongs to path $1$} \\
    0.5 - \Delta / K & \text{otherwise}\,,
  \end{cases}
\end{align*}
where $0 < \Delta / K < 0.5$. Note that our problem is designed such that $\Delta_{e, \min} = \Delta$ for any item $e$ in path $j > 1$.

The key observation is that our problem is equivalent to a $(L / K)$-arm Bernoulli bandit whose returns are scaled by $K$, when the learning agent knows that the weights of the items in the same path are equal. Therefore, we can derive lower bounds for our problem based on the existing lower bounds for Bernoulli bandits \cite{auer02nonstochastic,bubeck12regret,lai85asymptotically}.

Our gap-dependent lower bound is derived for the class of consistent algorithms, which is defined as follows. We say that the algorithm is \emph{consistent} if for any stochastic combinatorial semi-bandit, any suboptimal $A$, and any $\alpha > 0$, $\EE{T_n(A)} = o(n^\alpha)$, where $T_n(A)$ is the number of times that solution $A$ is chosen in $n$ steps. The restriction to the consistent algorithms is without loss of generality. In particular, an inconsistent algorithm is guaranteed to perform poorly on some semi-bandit, and therefore cannot achieve logarithmic regret on all semi-bandits, as $\combucb$.

\begin{proposition}
\label{prop:lower bound} For any $L$ and $K$ such that $L / K$ is an integer, and any $0 < \Delta / K < 0.5$, the regret of any consistent algorithm on the $K$-path semi-bandit problem is bounded from below as:
\begin{align*}
  \liminf_{n \to \infty} \frac{R(n)}{\log n} \geq
  \frac{(L - K) K}{4 \Delta}\,.
\end{align*}
\end{proposition}
\begin{proof}
The proposition is proved as follows:
\begin{align*}
  \liminf_{n \to \infty} \frac{R(n)}{\log n}
  & \stackrel{\text{(a)}}{\geq} K \sum_{k = 2}^{L / K} \frac{\Delta / K}{D_\mathrm{KL}(0.5 - \Delta / K \,\|\, 0.5)}
  \nonumber \\
  & = \left(\frac{L}{K} - 1\right) \frac{\Delta}{D_\mathrm{KL}(0.5 - \Delta / K \,\|\, 0.5)} \nonumber \\
  & \stackrel{\text{(b)}}{\geq} \frac{(L - K) K}{4 \Delta}\,,
\end{align*}
where $D_\mathrm{KL}(p \,\|\, q)$ is the Kullback-Leibler (KL) divergence between two Bernoulli random variables with means $p$ and $q$. Inequality (a) follows from the fact that our problem is equivalent to a $(L / K)$-arm Bernoulli bandit whose returns are scaled by $K$. Therefore, we can bound the regret from below using an existing lower bound for Bernoulli bandits \cite{lai85asymptotically}. Inequality (b) is due to $D_\mathrm{KL}(p \,\|\, q) \leq \frac{(p - q)^2}{q (1 - q)}$, where $p = 0.5 - \Delta / K$ and $q = 0.5$.
\end{proof}

We also derive a gap-free lower bound.

\begin{proposition}
\label{prop:gap-free lower bound} For any $L$ and $K$ such that $L / K$ is an integer, and any horizon $n > 0$, there exists a $K$-path semi-bandit problem such that the regret of any algorithm is:
\begin{align*}
  R(n) \geq \frac{1}{20} \min(\sqrt{K L n}, Kn)\,.
\end{align*}
\end{proposition}
\begin{proof}
The $K$-path semi-bandit problem is equivalent to a $(L / K)$-arm Bernoulli bandit whose payoffs are scaled by $K$. Therefore, we can apply Theorem 5.1 of Auer \etal~\cite{auer02nonstochastic} and bound the regret of any algorithm from below by:
\begin{align*}
  \frac{K}{20} \min(\sqrt{(L / K) n}, n) = \frac{1}{20} \min(\sqrt{K L n}, K n)\,.
\end{align*}
Note that the bound of Auer \etal~\cite{auer02nonstochastic} is for the adversarial setting. However, the worst-case environment in the proof is stochastic and therefore it applies to our problem.
\end{proof}

%!TEX root = Paper.tex

\section{Experiments}
\label{sec:experiments}

In this section, we evaluate $\combucb$ on a synthetic problem and demonstrate that its regret grows as suggested by our $O(K L (1 / \Delta) \log n)$ upper bound. We experiment with a stochastic longest-path problem on a $(m + 1) \times (m + 1)$ square grid (Figure~\ref{fig:main}b). The items in the ground set $E$ are the edges in the grid, $2 m (m + 1)$ in total. The feasible set $\Theta$ are all paths in the grid from the upper left corner to the bottom right corner that follow the directions of the edges. The length of these paths is $K = 2 m$. The weight of edge $e$ is drawn i.i.d. from a Bernoulli distribution with mean:
\begin{align*}
  \bar{w}(e) =
  \begin{cases}
    0.5 + \sigma / 2 & \text{$e$ is a leftmost or bottomost edge} \\
    0.5 - \sigma / 2 & \text{otherwise}\,,
  \end{cases}
\end{align*}
where $0 < \sigma < 1$. The optimal solution $A^\ast$ is a path along the leftmost and bottommost edges (Figure~\ref{fig:main}b).

The sample complexity of our problem is characterized by $|\tilde{E}| = 2 m (m + 1) - 2 m$ gaps $\Delta_{e, \min}$ ranging from $2 \sigma$ to $2 m \sigma$. It is easy to show that the number of items $e$ where $\Delta_{e, \min} = i \sigma$ is at most $2 (i - 1)$. Therefore, we can bound the $(\log n)$-term in Theorem~\ref{thm:K} as:
\begin{align}
  \sum_{e \in \tilde{E}} K \frac{534}{\Delta_{e, \min}} \log n
  & < 1068 m \sum_{i = 2}^{2 m} \frac{2 i}{i \sigma} \log n \nonumber \\
  & < \frac{4272 m^2 \log n}{\sigma}\,. \label{eq:grid regret bound}
\end{align}
Now we validate the dependence on $m$ and $\sigma$ empirically. We vary $m$ and $\sigma$, and run $\combucb$ for $n = 10^5$ steps.

Our experimental results are reported in Figure~\ref{fig:main}c. We observe two trends. First, the regret of $\combucb$ is linear in the number of items $L$, which depends quadratically on $m$ since $L = 2 m (m + 1)$. Second, the regret is linear in $1 / \sigma$. The dependence on $m$ and $1 / \sigma$ is the same as in our upper bound in \eqref{eq:grid regret bound}.

%!TEX root = Paper.tex

\section{Related Work}
\label{sec:related work}

Gai \etal~\cite{gai12combinatorial} proposed $\combucb$ and analyzed it. Chen \etal~\cite{chen13combinatorial} derived a $O(K^2 L (1 / \Delta) \log n)$ upper bound on the $n$-step regret of $\combucb$. In this paper, we show that the regret of $\combucb$ is $O(K L (1 / \Delta) \log n)$, a factor of $K$ improvement over the upper bound of Chen \etal~\cite{chen13combinatorial}. This upper bound is tight. We also prove a gap-free upper bound and show that it is nearly tight.

COMBAND \cite{cesabianchi12combinatorial}, online stochastic mirror descent (OSMD) \cite{audibert14regret}, and follow-the-perturbed-leader (FPL) with geometric resampling \cite{neu13efficient} are three recently proposed algorithms for adversarial combinatorial semi-bandits. In general, OSMD achieves optimal regret but is not guaranteed to be computationally efficient, in the same sense as in Section~\ref{sec:discussion}. FPL does not achieve optimal regret but is computationally efficient. It is an open problem whether adversarial combinatorial semi-bandits can be solved both computationally and sample efficiently. In this paper, we close this problem in the stochastic setting.

Matroid and polymatroid bandits \cite{kveton14matroid,kveton14learning} are instances of stochastic combinatorial semi-bandits. The $n$-step regret of $\combucb$ in these problems is $O(L (1 / \Delta) \log n)$, a factor of $K$ smaller than is suggested by our $O(K L (1 / \Delta) \log n)$ upper bound. However, we note that the bound of Kveton \etal~\cite{kveton14matroid,kveton14learning} is less general, as it applies only to matroids and polymatroids.

Our problem can be viewed as a linear bandit \cite{auer02using,abbasi-yadkori11improved}, where each solution $A$ is associated with an indicator vector $\bx \in \set{0, 1}^E$ and the learning agent observes the weight of each non-zero entry of $\bx$. This feedback model is clearly more informative than that in linear bandits, where the learning agent observes just the sum of the weights. Therefore, our learning problem has lower sample complexity. In particular, note that our $\Omega(\sqrt{K L n})$ lower bound (Proposition~\ref{prop:gap-free lower bound}) is $\sqrt{K}$ smaller than that of Audibert \etal~\cite{audibert14regret} (Theorem 5) for combinatorial linear bandits. The bound of Audibert \etal~\cite{audibert14regret} is proved for the adversarial setting. Nevertheless, it applies to our setting because the worst-case environment in the proof is stochastic.

Russo and Van Roy \cite{russo14information}, and Wen \etal~\cite{wen14efficient}, derived upper bounds on the \emph{Bayes regret} of \emph{Thompson sampling} in stochastic combinatorial semi-bandits. These bounds have a similar form as our gap-free upper bound in Theorem~\ref{thm:gap-free}. However, they differ from our work in two aspects. First, the Bayes regret is a different performance metric from regret. From the frequentist perspective, it is a much weaker metric. Second, we also derive $O(\log n)$ upper bounds.

%!TEX root = Paper.tex

\section{Extensions}
\label{sec:extensions}

The computational efficiency of $\combucb$ depends on the computational efficiency of the offline optimization oracle. When the oracle is inefficient, we suggest resorting to approximations. Let ${\tt ALG}$ be a computationally-efficient oracle that returns an approximation. Then $\combucb$ can be straightforwardly modified to call ${\tt ALG}$ instead of the original oracle. Moreover, it is easy to bound the regret of this algorithm if it is measured with respect to the best approximate solution by ${\tt ALG}$ in hindsight.

Thompson sampling \cite{thompson33likelihood} often performs better in practice than ${\tt UCB1}$ \cite{auer02finitetime}. It is straightforward to propose a variant of $\combucb$ that uses Thompson sampling, by replacing the UCBs in Algorithm~\ref{alg:ucb1} with sampling from the posterior on the mean of the weights. The frequentist analysis of regret in Thompson sampling \cite{agrawal12analysis} resembles the analysis of ${\tt UCB1}$. Therefore, we believe that our analysis can be generalized to Thompson sampling, and we hypothesize that the regret of the resulting algorithm is $O(K L (1 / \Delta) \log n)$.

%Finally, note that any feasible solution $A$ can be represented as a vector $\bx_A \in \set{0, 1}^E$ such that $\bx_A(e) = \I{e \in A}$. Our analysis does not rely on the assumption that $\bx_A(e)$ is integral. Therefore, it can be easily extended to the cases where $\bx_A \in [0, 1]^E$, such as learning variants of minimum-cost maximum flows \cite{papadimitriou98combinatorial}.

%!TEX root = Paper.tex

\section{Conclusions}
\label{sec:conclusions}

The main contribution of this work is that we derive novel gap-dependent and gap-free upper bounds on the regret of $\combucb$, a UCB-like algorithm for stochastic combinatorial semi-bandits. These bounds are tight up to polylogarithmic factors. In other words, we show that $\combucb$ is sample efficient because it achieves near-optimal regret. It is well known that $\combucb$ is also computationally efficient \cite{gai12combinatorial}, it can be implemented efficiently whenever the offline variant of the problem can be solved computationally efficiently. Therefore, we indirectly show that stochastic combinatorial semi-bandits can be solved both computationally and sample efficiently, by $\combucb$.

Theorems~\ref{thm:K one gap} and \ref{thm:K} are proved quite generally, for any $(\alpha_i)$ and $(\beta_i)$ subject to relatively mild constraints. At the end of the proofs, we choose $(\alpha_i)$ and $(\beta_i)$ to be geometric sequences. This is sufficient for our purpose. But the choice is likely to be suboptimal and may lead to larger constants in our upper bounds than is necessary. We leave the problem of choosing better $(\alpha_i)$ and $(\beta_i)$ for future work.

We leave open several questions of interest. For instance, our $\Omega(K L (1 / \Delta) \log n)$ lower bound is derived on a problem where all suboptimal solutions have the same gaps. So technically speaking, our $O(K L (1 / \Delta) \log n)$ upper bound is tight only on this family of problems. It is an open problem whether our upper bound is tight in general.

Our $O(\sqrt{K L n \log n})$ upper bound matches the $\Omega(\sqrt{K L n})$ lower bound up to a factor of $\sqrt{\log n}$. We believe that this factor can be eliminated by modifying the confidence radii in $\combucb$ \eqref{eq:confidence radius} along the lines of Audibert \etal~\cite{audibert09minimax}. We leave this for future work.

\bibliographystyle{plain}
\bibliography{References}

%!TEX root = Paper.tex

\clearpage
\onecolumn
\appendix

\section{Proofs of Main Theorems}
\label{sec:proofs}

\subsection{Proof of Lemma~\ref{lem:regret decomposition}}
\label{sec:proof regret decomposition}

Let $R_t = R(A_t, w_t)$ be the stochastic regret of $\combucb$ at time $t$, where $A_t$ and $w_t$ are the solution and the weights of the items at time $t$, respectively. Furthermore, let $\cE_t = \set{\exists e \in E: \abs{\bar{w}(e) - \hat{w}_{T_{t - 1}(e)}(e)} \geq c_{t - 1, T_{t - 1}(e)}}$ be the event that $\bar{w}(e)$ is outside of the high-probability confidence interval around $\hat{w}_{T_{t - 1}(e)}(e)$ for some item $e$ at time $t$; and let $\ccE_t$ be the complement of $\cE_t$, $\bar{w}(e)$ is in the high-probability confidence interval around $\hat{w}_{T_{t - 1}(e)}(e)$ for all $e$ at time $t$. Then we can decompose the regret of $\combucb$ as:
\begin{align*}
  R(n) = \EE{\sum_{t = 1}^{t_0 - 1} R_t} +
  \EE{\sum_{t = t_0}^n \I{\cE_t} R_t} +
  \EE{\sum_{t = t_0}^n \I{\ccE_t} R_t}\,.
\end{align*}
Now we bound each term in our regret decomposition.

The regret of the initialization, $\EE{\sum_{t = 1}^{t_0 - 1} R_t}$, is bounded by $K L$ because Algorithm~\ref{alg:initialization} terminates in at most $L$ steps, and $R_t \leq K$ for any $A_t$ and $w_t$.

The second term in our regret decomposition, $\EE{\sum_{t = t_0}^n \I{\cE_t} R_t}$, is small because all of our confidence intervals hold with high probability. In particular, for any $e$, $s$, and $t$:
\begin{align*}
  P(\abs{\bar{w}(e) - \hat{w}_s(e)} \geq c_{t, s}) \leq 2 \exp[-3 \log t]\,,
\end{align*}
and therefore:
\begin{align*}
  \EE{\sum_{t = t_0}^n \I{\cE_t}} \leq
  \sum_{e \in E} \sum_{t = 1}^n \sum_{s = 1}^t P(\abs{\bar{w}(e) - \hat{w}_s(e)} \geq c_{t, s}) \leq
  2 \sum_{e \in E} \sum_{t = 1}^n \sum_{s = 1}^t \exp[-3 \log t] \leq
  2 \sum_{e \in E} \sum_{t = 1}^n t^{-2} \leq
  \frac{\pi^2}{3} L\,.
\end{align*}
Since $R_t \leq K$ for any $A_t$ and $w_t$, $\EE{\sum_{t = t_0}^n \I{\cE_t} R_t} \leq \frac{\pi^2}{3} K L$.

Finally, we rewrite the last term in our regret decomposition as:
\begin{align*}
  \EE{\sum_{t = t_0}^n \I{\ccE_t} R_t} \stackrel{\text{(a)}}{=}
  \sum_{t = t_0}^n \EE{\I{\ccE_t} \EE{R_t \,|\, A_t}} \stackrel{\text{(b)}}{=}
  \EE{\sum_{t = t_0}^n \Delta_{A_t} \I{\ccE_t, \Delta_{A_t} > 0}}\,.
\end{align*}
In equality (a), the outer expectation is over the history of the agent up to time $t$, which in turn determines $A_t$ and $\ccE_t$; and $\EE{R_t \,|\, A_t}$ is the expected regret at time $t$ conditioned on solution $A_t$. Equality (b) follows from $\Delta_{A_t} = \EE{R_t \,|\, A_t}$. Now we bound $\Delta_{A_t} \I{\ccE_t, \Delta_{A_t} > 0}$ for any suboptimal $A_t$. The bound is derived based on two facts. First, when $\combucb$ chooses $A_t$, $f(A_t, U_t) \geq f(A^\ast, U_t)$. This further implies that $\sum_{e \in A_t \setminus A^\ast} U_t(e) \geq \sum_{e \in A^\ast \setminus A_t} U_t(e)$. Second, when event $\ccE_t$ happens, $\abs{\bar{w}(e) - \hat{w}_{T_{t - 1}(e)}(e)} < c_{t - 1, T_{t - 1}(e)}$ for all items $e$. Therefore:
\begin{align*}
  \sum_{e \in A_t \setminus A^\ast} \hspace{-0.1in} \bar{w}(e) +
  2 \hspace{-0.1in} \sum_{e \in A_t \setminus A^\ast} \hspace{-0.1in} c_{t - 1, T_{t - 1}(e)} \geq
  \sum_{e \in A_t \setminus A^\ast} \hspace{-0.1in} U_t(e) \geq
  \sum_{e \in A^\ast \setminus A_t} \hspace{-0.1in} U_t(e) \geq
  \sum_{e \in A^\ast \setminus A_t} \hspace{-0.1in} \bar{w}(e)\,,
\end{align*}
and $2 \sum_{e \in A_t \setminus A^\ast} c_{t - 1, T_{t - 1}(e)} \geq \Delta_{A_t}$ follows from the observation
that $\Delta_{A_t} =  \sum_{e \in A^\ast \setminus A_t} \bar{w}(e) - \sum_{e \in A_t \setminus A^\ast} \bar{w}(e)$. Now note that $c_{n, T_{t - 1}(e)} \geq c_{t - 1, T_{t - 1}(e)}$ for any time $t \leq n$. Therefore, the event $\cF_t$ in \eqref{eq:suboptimality event} must happen and:
\begin{align*}
  \EE{\sum_{t = t_0}^n \Delta_{A_t} \I{\ccE_t, \Delta_{A_t} > 0}} \leq
  \EE{\sum_{t = t_0}^n \Delta_{A_t} \I{\cF_t}}\,.
\end{align*}
This concludes our proof.

\subsection{Proof of Theorem~\ref{thm:K43 one gap}}
\label{sec:proof K43 one gap}

By Lemma~\ref{lem:regret decomposition}, it remains to bound $\hat{R}(n) = \sum_{t = t_0}^n \Delta_{A_t} \I{\cF_t}$, where the event $\cF_t$ is defined in \eqref{eq:suboptimality event}. By Lemma~\ref{lem:K43 events} and from the assumption that $\Delta_{A_t} = \Delta$ for all suboptimal $A_t$, it follows that:
\begin{align*}
  \hat{R}(n) =
  \Delta \sum_{t = t_0}^n \I{\cF_t} =
  \Delta \sum_{t = t_0}^n \I{G_{1, t}, \Delta_{A_t} > 0} +
  \Delta \sum_{t = t_0}^n \I{G_{2, t}, \Delta_{A_t} > 0}\,.
\end{align*}
To bound the above quantity, it is sufficient to bound the number of times that events $G_{1, t}$ and $G_{2, t}$ happen. Then we set the tunable parameters $d$ and $\alpha$ such that the two counts are of the same magnitude.

\begin{claim}
\label{claim:event 1} Event $G_{1, t}$ happens at most $\displaystyle \frac{\alpha}{d} K^2 L \frac{6}{\Delta^2} \log n$ times.
\end{claim}
%%%%%     Proof     %%%%%
\begin{proof}
Recall that event $G_{1, t}$ can happen only if at least $d$ chosen suboptimal items are not observed ``sufficiently often'' up to time $t$, $T_{t - 1}(e) \leq \alpha K^2 \frac{6}{\Delta^2} \log n$ for at least $d$ items in $\tilde{A}_t$. After the event happens, the observation counters of these items increase by one. Therefore, after the event happens $\frac{\alpha}{d} K^2 L \frac{6}{\Delta^2} \log n$ times, all suboptimal items are guaranteed to be observed at least $\alpha K^2 \frac{6}{\Delta^2} \log n$ times and $G_{1, t}$ cannot happen anymore.
\end{proof}

\begin{claim}
\label{claim:event 2} Event $G_{2, t}$ happens at most $\displaystyle \frac{\alpha d^2}{(\sqrt{\alpha} - 1)^2} L \frac{6}{\Delta^2} \log n$ times.
\end{claim}
%%%%%     Proof     %%%%%
\begin{proof}
Event $G_{2, t}$ can happen only if there exists $e \in \tilde{A}_t$ such that $T_{t - 1}(e) \leq \frac{\alpha d^2}{(\sqrt{\alpha} - 1)^2} \frac{6}{\Delta^2} \log n$. After the event happens, the observation counter of item $e$ increases by one. Therefore, the number of times that event $G_{2, t}$ can happen is bounded trivially by $\frac{\alpha d^2}{(\sqrt{\alpha} - 1)^2} L \frac{6}{\Delta^2} \log n$.
\end{proof}

\noindent Based on Claims~\ref{claim:event 1} and \ref{claim:event 2}, $\hat{R}(n)$ is bounded as:
\begin{align*}
  \hat{R}(n) \leq
  \left(\frac{\alpha}{d} K^2 + \frac{\alpha d^2}{(\sqrt{\alpha} - 1)^2}\right) L \frac{6}{\Delta} \log n\,.
\end{align*}
Finally, we choose $\alpha = 4$ and $d = K^\frac{2}{3}$; and it follows that the regret is bounded as:
\begin{align*}
  R(n) \leq
  \EE{\hat{R}(n)} + \left(\frac{\pi^2}{3} + 1\right) K L \leq
  K^\frac{4}{3} L \frac{48}{\Delta} \log n + \left(\frac{\pi^2}{3} + 1\right) K L\,.
\end{align*}

\subsection{Proof of Theorem~\ref{thm:K43}}
\label{sec:proof K43}

Let $\cF_t$ be the event in \eqref{eq:suboptimality event}. By Lemmas~\ref{lem:regret decomposition} and \ref{lem:K43 events}, it remains to bound:
\begin{align*}
  \hat{R}(n) =
  \sum_{t = t_0}^n \Delta_{A_t} \I{\cF_t} =
  \sum_{t = t_0}^n \Delta_{A_t} \I{G_{1, t}, \Delta_{A_t} > 0} +
  \sum_{t = t_0}^n \Delta_{A_t} \I{G_{2, t}, \Delta_{A_t} > 0}\,.
\end{align*}
In the next step, we introduce item-specific variants of events $G_{1, t}$ \eqref{eq:event 1} and $G_{2, t}$ \eqref{eq:event 2}, and then associate the regret at time $t$ with these events. In particular, let:
\begin{align}
  G_{e, 1, t} & = G_{1, t} \cap \left\{e \in \tilde{A}_t,
  T_{t - 1}(e) \leq \alpha K^2 \frac{6}{\Delta_{A_t}^2} \log n\right\} \\
  G_{e, 2, t} & = G_{2, t} \cap \left\{e \in \tilde{A}_t,
  T_{t - 1}(e) \leq \frac{\alpha d^2}{(\sqrt{\alpha} - 1)^2} \frac{6}{\Delta_{A_t}^2} \log n\right\}
\end{align}
be the events that item $e$ is not observed ``sufficiently often'' under events $G_{1, t}$ and $G_{2, t}$, respectively. Then by the definitions of the above events, it follows that:
\begin{align*}
  \I{G_{1, t}, \Delta_{A_t} > 0} & \leq
  \frac{1}{d} \sum_{e \in \tilde{E}} \I{G_{e, 1, t}, \Delta_{A_t} > 0} \\
  \I{G_{2, t}, \Delta_{A_t} > 0} &
  \leq \sum_{e \in \tilde{E}} \I{G_{e, 2, t}, \Delta_{A_t} > 0}\,,
\end{align*}
where $\tilde{E} = E \setminus A^\ast$ is the set of subptimal items; and we bound $\hat{R}(n)$ as:
\begin{align*}
  \hat{R}(n) \leq
  \sum_{e \in \tilde{E}} \sum_{t = t_0}^n
  \I{G_{e, 1, t}, \Delta_{A_t} > 0} \frac{\Delta_{A_t}}{d} +
  \sum_{e \in \tilde{E}} \sum_{t = t_0}^n
  \I{G_{e, 2, t}, \Delta_{A_t} > 0} \Delta_{A_t}\,.
\end{align*}
Let each item $e$ be contained in $N_e$ suboptimal solutions and $\Delta_{e, 1} \geq \ldots \geq \Delta_{e, N_e}$ be the gaps of these solutions, ordered from the largest gap to the smallest one. Then $\hat{R}(n)$ can be further bounded as:
\begin{align*}
  \hat{R}(n)
  & \leq \sum_{e \in \tilde{E}} \sum_{t = t_0}^n \sum_{k = 1}^{N_e}
  \I{G_{e, 1, t}, \Delta_{A_t} = \Delta_{e, k}} \frac{\Delta_{e, k}}{d} +
  \sum_{e \in \tilde{E}} \sum_{t = t_0}^n \sum_{k = 1}^{N_e}
  \I{G_{e, 2, t}, \Delta_{A_t} = \Delta_{e, k}} \Delta_{e, k} \\
  & \stackrel{\text{(a)}}{\leq} \sum_{e \in \tilde{E}} \sum_{t = t_0}^n \sum_{k = 1}^{N_e}
  \I{e \in \tilde{A}_t, T_{t - 1}(e) \leq \alpha K^2 \frac{6}{\Delta_{e, k}^2} \log n,
  \Delta_{A_t} = \Delta_{e, k}} \frac{\Delta_{e, k}}{d} + {} \\
  & \quad\, \sum_{e \in \tilde{E}} \sum_{t = t_0}^n \sum_{k = 1}^{N_e}
  \I{e \in \tilde{A}_t, T_{t - 1}(e) \leq \frac{\alpha d^2}{(\sqrt{\alpha} - 1)^2} \frac{6}{\Delta_{e, k}^2} \log n,
  \Delta_{A_t} = \Delta_{e, k}} \Delta_{e, k} \\
  & \stackrel{\text{(b)}}{\leq} \sum_{e \in \tilde{E}}
  \frac{6 \alpha K^2 \log n}{d}
  \left[\Delta_{e, 1} \frac{1}{\Delta_{e, 1}^2} + \sum_{k = 2}^{N_e} \Delta_{e, k}
  \left(\frac{1}{\Delta_{e, k}^2} - \frac{1}{\Delta_{e, k - 1}^2}\right)\right] + {} \\
  & \quad\, \sum_{e \in \tilde{E}}
  \frac{6 \alpha d^2 \log n}{(\sqrt{\alpha} - 1)^2}
  \left[\Delta_{e, 1} \frac{1}{\Delta_{e, 1}^2} + \sum_{k = 2}^{N_e} \Delta_{e, k}
  \left(\frac{1}{\Delta_{e, k}^2} - \frac{1}{\Delta_{e, k - 1}^2}\right)\right] \\
  & \stackrel{\text{(c)}}{<} \sum_{e \in \tilde{E}}
  \left(\frac{\alpha}{d} K^2 + \frac{\alpha d^2}{(\sqrt{\alpha} - 1)^2}\right)
  \frac{12}{\Delta_{e, \min}} \log n\,,
\end{align*}
where inequality (a) is by the definitions of events $G_{e, 1, t}$ and $G_{e, 2, t}$, inequality (b) is from the solution to:
\begin{align*}
  \max_{A_1, \dots, A_n} \sum_{t = t_0}^n \sum_{k = 1}^{N_e}
  \I{e \in \tilde{A}_t, T_{t - 1}(e) \leq \frac{C}{\Delta_{e, k}^2} \log n,
  \Delta_{A_t} = \Delta_{e, k}} \Delta_{e, k}
\end{align*}
for appropriate $C$, and inequality (c) follows from Lemma 3 of Kveton \etal~\cite{kveton14matroid}:
\begin{align}
  \left[\Delta_{e, 1} \frac{1}{\Delta_{e, 1}^2} + \sum_{k = 2}^{N_e} \Delta_{e, k}
  \left(\frac{1}{\Delta_{e, k}^2} - \frac{1}{\Delta_{e, k - 1}^2}\right)\right] <
  \frac{2}{\Delta_{e, N_e}} =
  \frac{2}{\Delta_{e, \min}}\,.
  \label{eq:kveton2014}
\end{align}
Finally, we choose $\alpha = 4$ and $d = K^\frac{2}{3}$; and it follows that the regret is bounded as:
\begin{align*}
  R(n) \leq
  \EE{\hat{R}(n)} + \left(\frac{\pi^2}{3} + 1\right) K L \leq
  \sum_{e \in \tilde{E}} K^\frac{4}{3} \frac{96}{\Delta_{e, \min}} \log n +
  \left(\frac{\pi^2}{3} + 1\right) K L\,.
\end{align*}

\subsection{Proof of Theorem~\ref{thm:K one gap}}
\label{sec:proof K one gap}

The first step of the proof is identical to that of Theorem~\ref{thm:K43 one gap}. By Lemma~\ref{lem:regret decomposition}, it remains to bound $\hat{R}(n) = \sum_{t = t_0}^n \Delta_{A_t} \I{\cF_t}$, where the event $\cF_t$ is defined in \eqref{eq:suboptimality event}. By Lemma~\ref{lem:complete system} and from the assumption that $\Delta_{A_t} = \Delta$ for all suboptimal $A_t$, it follows that:
\begin{align*}
  \hat{R}(n) =
  \Delta \sum_{t = t_0}^n \I{\cF_t} =
  \Delta \sum_{i = 1}^\infty \sum_{t = t_0}^n \I{G_{i, t}, \Delta_{A_t} > 0}\,.
\end{align*}
Note that $\Delta_{A_t} > 0$ implies $\Delta_{A_t} = \Delta$. Therefore, $m_{i, t}$ does not depend on $t$ and we denote it by $m_i = \alpha_i \frac{K^2}{\Delta^2} \log n$. Based on the same argument as in Claim~\ref{claim:event 1}, event $G_{i, t}$ cannot happen more than $\frac{L m_i}{\beta_i K}$ times, because at least $\beta_i K$ items that are observed at most $m_i$ times have their observation counters incremented in each event $G_{i, t}$. Therefore:
\begin{align}
  \hat{R}(n) \leq
  \Delta \sum_{i = 1}^\infty \frac{L m_i}{\beta_i K} =
  K L \frac{1}{\Delta} \left[\sum_{i = 1}^\infty \frac{\alpha_i}{\beta_i}\right] \log n\,.
  \label{eq:ab regret bound}
\end{align}
It remains to choose $(\alpha_i)$ and $(\beta_i)$ such that:
\begin{itemize}
  \item $\lim_{i \to \infty} \alpha_i = \lim_{i \to \infty} \beta_i = 0$;
  \item Monotonicity conditions in \eqref{eq:condition 1} and \eqref{eq:condition 2} hold;
  \item Inequality~\eqref{eq:condition 3} holds,
  $\sqrt{6} \sum_{i = 1}^\infty \frac{\beta_{i - 1} - \beta_i}{\sqrt{\alpha_i}} \leq 1$;
  \item $\sum_{i = 1}^\infty \frac{\alpha_i}{\beta_i}$ is minimized.
\end{itemize}
We choose $(\alpha_i)$ and $(\beta_i)$ to be geometric sequences, $\beta_i = \beta^i$ and $\alpha_i = d \alpha^i$ for $0 < \alpha, \beta < 1$ and $d > 0$. For this setting, $\alpha_i \to 0$ and $\beta_i \to 0$, and the monotonicity conditions are also satisfied. Moreover, if $\beta < \sqrt{\alpha}$, we have:
\begin{align*}
  \sqrt{6} \sum_{i = 1}^\infty \frac{\beta_{i - 1} - \beta_i}{\sqrt{\alpha_i}} =
  \sqrt{6} \sum_{i = 1}^\infty \frac{\beta^{i - 1} - \beta^i}{\sqrt{d \alpha^i}} =
  \sqrt{\frac{6}{d}} \frac{1 - \beta}{\sqrt{\alpha} - \beta} \leq
  1
\end{align*}
provided that $d \geq 6 \left(\frac{1 - \beta}{\sqrt{\alpha} - \beta}\right)^2$. Furthermore, if $\alpha < \beta$, we have:
\begin{align*}
  \sum_{i = 1}^\infty \frac{\alpha_i}{\beta_i} =
  \sum_{i = 1}^\infty \frac{d \alpha^i}{\beta^i} = 
  \frac{d \alpha}{\beta - \alpha}\,.
\end{align*}
Given the above, the best choice of $d$ is $6 \left(\frac{1 - \beta}{\sqrt{\alpha} - \beta}\right)^2$ and the problem of minimizing the constant in our regret bound can be written as:
\begin{align*}
  \inf_{\alpha, \beta} & \quad 6 \left(\frac{1 - \beta}{\sqrt{\alpha} - \beta}\right)^2
  \frac{\alpha}{\beta - \alpha} \\
  \textrm{s.t.} & \quad 0 < \alpha < \beta < \sqrt{\alpha} < 1\,.
\end{align*}
We find the solution to the above problem numerically, and determine it to be $\alpha = 0.1459$ and $\beta = 0.2360$. For these $\alpha$ and $\beta$, $6 \left(\frac{1 - \beta}{\sqrt{\alpha} - \beta}\right)^2 \frac{\alpha}{\beta - \alpha} < 267$. We apply this upper bound to \eqref{eq:ab regret bound} and it follows that the regret is bounded as:
\begin{align*}
  R(n) \leq
  \EE{\hat{R}(n)} + \left(\frac{\pi^2}{3} + 1\right) K L \leq
  K L \frac{267}{\Delta} \log n + \left(\frac{\pi^2}{3} + 1\right) K L\,.
\end{align*}

\subsection{Proof of Theorem~\ref{thm:K}}
\label{sec:proof K}

Let $\cF_t$ be the event in \eqref{eq:suboptimality event}. By Lemmas~\ref{lem:regret decomposition} and \ref{lem:complete system}, it remains to bound:
\begin{align*}
  \hat{R}(n) =
  \sum_{t = t_0}^n \Delta_{A_t} \I{\cF_t} =
  \sum_{i = 1}^\infty \sum_{t = t_0}^n \Delta_{A_t} \I{G_{i, t}, \Delta_{A_t} > 0}\,.
\end{align*}
In the next step, we define item-specific variants of events $G_{i, t}$ \eqref{eq:event i} and associate the regret at time $t$ with these events. In particular, let:
\begin{align}
  G_{e, i, t} = G_{i, t} \cap \left\{e \in \tilde{A}_t, T_{t - 1}(e) \leq m_{i, t}\right\}
\end{align}
be the event that item $e$ is not observed ``sufficiently often'' under event $G_{i, t}$. Then it follows that:
\begin{align*}
  \I{G_{i, t}, \Delta_{A_t} > 0} \leq
  \frac{1}{\beta_i K} \sum_{e \in \tilde{E}} \I{G_{e, i, t}, \Delta_{A_t} > 0}\,,
\end{align*}
because at least $\beta_i K$ items are not observed ``sufficiently often'' under event $G_{i, t}$. Therefore, we can bound $\hat{R}(n)$ as:
\begin{align*}
  \hat{R}(n) \leq
  \sum_{e \in \tilde{E}} \sum_{i = 1}^\infty \sum_{t = t_0}^n
  \I{G_{e, i, t}, \Delta_{A_t} > 0} \frac{\Delta_{A_t}}{\beta_i K}\,.
\end{align*}
Let each item $e$ be contained in $N_e$ suboptimal solutions and $\Delta_{e, 1} \geq \ldots \geq \Delta_{e, N_e}$ be the gaps of these solutions, ordered from the largest gap to the smallest one. Then $\hat{R}(n)$ can be further bounded as:
\begin{align*}
  \hat{R}(n)
  & \leq \sum_{e \in \tilde{E}} \sum_{i = 1}^\infty \sum_{t = t_0}^n \sum_{k = 1}^{N_e}
  \I{G_{e, i, t}, \Delta_{A_t} = \Delta_{e, k}} \frac{\Delta_{e, k}}{\beta_i K} \\
  & \stackrel{\text{(a)}}{\leq} \sum_{e \in \tilde{E}} \sum_{i = 1}^\infty \sum_{t = t_0}^n \sum_{k = 1}^{N_e}
  \I{e \in \tilde{A}_t, T_{t - 1}(e) \leq \alpha_i \frac{K^2}{\Delta_{e, k}^2} \log n,
  \Delta_{A_t} = \Delta_{e, k}} \frac{\Delta_{e, k}}{\beta_i K} \\
  & \stackrel{\text{(b)}}{\leq} \sum_{e \in \tilde{E}} \sum_{i = 1}^\infty
  \frac{\alpha_i K \log n}{\beta_i}
  \left[\Delta_{e, 1} \frac{1}{\Delta_{e, 1}^2} + \sum_{k = 2}^{N_e} \Delta_{e, k}
  \left(\frac{1}{\Delta_{e, k}^2} - \frac{1}{\Delta_{e, k - 1}^2}\right)\right] \\
  & \stackrel{\text{(c)}}{<} \sum_{e \in \tilde{E}} \sum_{i = 1}^\infty
  \frac{\alpha_i K \log n}{\beta_i} \frac{2}{\Delta_{e, \min}} \\
  & = \sum_{e \in \tilde{E}} K \frac{2}{\Delta_{e, \min}}
  \left[\sum_{i = 1}^\infty \frac{\alpha_i}{\beta_i}\right] \log n\,,
\end{align*}
where inequality (a) is by the definition of event $G_{e, i, t}$, inequality (b) follows from the solution to:
\begin{align*}
  \max_{A_1, \dots, A_n} \sum_{t = t_0}^n \sum_{k = 1}^{N_e}
  \I{e \in \tilde{A}_t, T_{t - 1}(e) \leq \alpha_i \frac{K^2}{\Delta_{e, k}^2} \log n,
  \Delta_{A_t} = \Delta_{e, k}} \frac{\Delta_{e, k}}{\beta_i K}\,,
\end{align*}
and inequality (c) follows from \eqref{eq:kveton2014}. For the same $(\alpha_i)$ and $(\beta_i)$ as in Theorem~\ref{thm:K one gap}, we have $\sum_{i = 1}^\infty \frac{\alpha_i}{\beta_i} < 267$ and it follows that the regret is bounded as:
\begin{align*}
  R(n) \leq
  \EE{\hat{R}(n)} + \left(\frac{\pi^2}{3} + 1\right) K L \leq
  \sum_{e \in \tilde{E}} K \frac{534}{\Delta_{e, \min}} \log n +
  \left(\frac{\pi^2}{3} + 1\right) K L\,.
\end{align*}

\subsection{Proof of Theorem~\ref{thm:gap-free}}
\label{sec:proof gap-free}

The key idea is to decompose the regret of $\combucb$ into two parts, where the gaps are larger than $\epsilon$ and at most $\epsilon$. We analyze each part separately and then set $\epsilon$ to get the desired result.

By Lemma~\ref{lem:regret decomposition}, it remains to bound $\hat{R}(n) = \sum_{t = t_0}^n \Delta_{A_t} \I{\cF_t}$, where the event $\cF_t$ is defined in \eqref{eq:suboptimality event}. We partition $\hat{R}(n)$ as:
\begin{align*}
  \hat{R}(n)
  & = \sum_{t = t_0}^n \Delta_{A_t} \I{\cF_t, \Delta_{A_t} < \epsilon} +
  \sum_{t = t_0}^n \Delta_{A_t} \I{\cF_t, \Delta_{A_t} \geq \epsilon} \\
  & \leq \epsilon n + \sum_{t = t_0}^n \Delta_{A_t} \I{\cF_t, \Delta_{A_t} \geq \epsilon}
\end{align*}
and bound the first term trivially. The second term is bounded in the same way as $\hat{R}(n)$ in the proof of Theorem~\ref{thm:K}, except that we only consider the gaps $\Delta_{e, k} \geq \epsilon$. Therefore, $\Delta_{e, \min} \geq \epsilon$ and we get:
\begin{align*}
  \sum_{t = t_0}^n \Delta_{A_t} \I{\cF_t, \Delta_{A_t} \geq \epsilon} \leq
  \sum_{e \in \tilde{E}} K \frac{534}{\epsilon} \log n \leq
  K L \frac{534}{\epsilon} \log n\,.
\end{align*}
Based on the above inequalities:
\begin{align*}
  R(n) \leq \frac{534 K L}{\epsilon} \log n + \epsilon n + \left(\frac{\pi^2}{3} + 1\right) K L\,.
\end{align*}
Finally, we choose $\displaystyle \epsilon = \sqrt{\frac{534 K L \log n}{n}}$ and get:
\begin{align*}
  R(n) \leq
  2 \sqrt{534 K L n \log n} + \left(\frac{\pi^2}{3} + 1\right) K L <
  47 \sqrt{K L n \log n} + \left(\frac{\pi^2}{3} + 1\right) K L\,,
\end{align*}
which concludes our proof.

\section{Technical Lemmas}
\label{sec:lemmas}

\begin{lemma}
\label{lem:little helper} Let $S_i$, $\bar{S}_i$, and $m_i$ be defined as in Lemma~\ref{lem:complete system}; and $|S_i| < \beta_i K$ for all $i > 0$. Then:
\begin{align*}
  \sum_{i = 1}^\infty \frac{|\bar{S}_i \setminus \bar{S}_{i - 1}|}{\sqrt{m_i}} <
  \sum_{i = 1}^\infty \frac{(\beta_{i - 1} - \beta_i) K}{\sqrt{m_i}}\,.
\end{align*}
\end{lemma}
%%%%%     Proof     %%%%%
\begin{proof}
The lemma is proved as:
\begin{align*}
  \sum_{i = 1}^\infty |\bar{S}_i \setminus \bar{S}_{i - 1}| \frac{1}{\sqrt{m_i}}
  & = \sum_{i = 1}^\infty (|S_{i - 1} \setminus S_i|) \frac{1}{\sqrt{m_i}} \\
  & = \sum_{i = 1}^\infty (|S_{i - 1}| - |S_i|) \frac{1}{\sqrt{m_i}} \\
  & = \frac{|S_0|}{\sqrt{m_1}} +
  \sum_{i = 1}^\infty |S_i| \left(\frac{1}{\sqrt{m_{i + 1}}} - \frac{1}{\sqrt{m_i}}\right) \\
  & < \frac{\beta_0 K}{\sqrt{m_1}} +
  \sum_{i = 1}^\infty \beta_i K \left(\frac{1}{\sqrt{m_{i + 1}}} - \frac{1}{\sqrt{m_i}}\right) \\
  & = \sum_{i = 1}^\infty (\beta_{i - 1} - \beta_i) K \frac{1}{\sqrt{m_i}}\,.
\end{align*}
The first two equalities follow from the definitions of $\bar{S}_i$ and $S_i$. The inequality follows from the facts that $|S_i| < \beta_i K$ for all $i > 0$ and $|S_0| \leq \beta_0 K$.
\end{proof}

\end{document}